\documentclass[10pt]{article}

\usepackage{fullpage}
\usepackage{amsmath} 
\usepackage{amsthm}
\usepackage{mathtools}
\usepackage{graphicx}
\mathtoolsset{showonlyrefs=true}

\newtheorem{theorem}{Theorem}
\newtheorem{proposition}{Proposition}

\newcommand\spn[1]{{\left\|#1\right\|_{sp}}}

\def\G{{\cal G}}

\title{On the Use of Non-Stationary Policies for \\ Infinite-Horizon Discounted Markov Decision Processes}

\author{Bruno Scherrer}

\begin{document}

\maketitle

\begin{abstract}
We consider infinite-horizon $\gamma$-discounted Markov Decision Processes, for which it is known that there exists a stationary optimal policy. We consider the algorithm Value Iteration and the sequence of policies $\pi_1,\dots,\pi_k$ it implicitely generates until some iteration $k$. We provide performance bounds for non-stationary policies involving the last $m$ generated policies that reduce the state-of-the-art bound for the last stationary policy $\pi_k$ by a factor $\frac{1-\gamma}{1-\gamma^m}$. In particular, the use of non-stationary policies allows to reduce the usual asymptotic performance bounds of Value Iteration with errors bounded by $\epsilon$ at each iteration from $\frac{\gamma}{(1-\gamma)^2}\epsilon$ to $\frac{\gamma}{1-\gamma}\epsilon$, which is significant in the usual situation when $\gamma$ is close to $1$. Given Bellman operators that can only be computed with some error $\epsilon$, a surprising consequence of this result is that the problem of ``computing an approximately optimal non-stationary policy'' is much simpler than that of ``computing an approximately optimal stationary policy'', and even slightly simpler than that of ``approximately computing the value of some fixed policy'', since this last problem \emph{only} has a guarantee of $\frac{1}{1-\gamma}\epsilon$.
\end{abstract}

\vspace{1cm}
Given a Markov Decision Process, suppose on runs an approximate version of Value Iteration, that is one builds a sequence of value-policy pairs as follows:
\begin{align}
&\mbox{Pick any } \pi_{k+1} \mbox{ in } \G v_k\\
&v_{k+1}= T_{\pi_{k+1}} v_k + \epsilon_{k+1}
\end{align}
where $v_0$ is arbitrary, $\G v_k$ is the set of policies that are greedy\footnote{There may be several greedy policies with respect to some value $v$, and what we write here holds whichever one is picked.} with respect to $v_k$,  and $T_{\pi_k}$ is the linear Bellman operator associated to policy $\pi_k$.
Though it does not appear exactly in this form in the literature, the following performance bound is somewhat standard. 
\begin{theorem}
\label{refbound}
Let $\epsilon=\max_{1 \le j < k}\spn{\epsilon_j}$ be a uniform upper bound on the span seminorm\footnote{For any function $f$ defined on the state space, the span seminorm of $f$ is $\spn{f}=\max_s f(s)-\min_s f(s)$. The motivation for using the span seminorm instead of a more usual $L_\infty$-norm is twofold: 1) it slightly improves on the state-of-the-art bounds and 2) it simplifies the construction of an example in the proof of the forthcoming Proposition~\ref{tightbound}.} of the errors before iteration $k$.
The loss of policy $\pi_k$ is bounded as follows:
\begin{align}
\|v_*-v_{\pi_k}\|_\infty \le \frac{1}{1-\gamma} \left( \frac{\gamma-\gamma^k}{1-\gamma}\epsilon + \gamma^k \spn{v_*-v_0} \right) \label{bound1}.
\end{align}
\end{theorem}
In Theorem~\ref{main}, we will prove a generalization of this result, so we do not provide a proof here.
Since for any $f$,  $\spn{f} \le 2 \|f\|_\infty$, Theorem~\ref{refbound} constitutes a slight improvement and a (finite-iteration) generalization of the following well-known performance bound (see \cite{ndp}):
\begin{align}
\lim\sup_{k \rightarrow \infty} \|v_*-v_{\pi_k}\|_\infty \le \frac{2\gamma}{(1-\gamma)^2} \max_k \|\epsilon_k\|_\infty.
\end{align}
Asymptotically, the above bounds involve a
$\frac{\gamma}{(1-\gamma)^2}$ constant that may be really big when
$\gamma$ is close to~1. Compared to a value-iteration algorithm for
approximately computing the value of some fixed policy, and for which
one can prove a dependency of the form $\frac{1}{1-\gamma}\epsilon$,
there is an extra term $\frac{\gamma}{1-\gamma}$ that suggests that
the problem of ``computing an approximately optimal policy'' is
significantly harder than that of ``approximately computing the
  value of some fixed policy''. To our knowledge, there does not exist
any example in the literature that supports the tightness of the above
mentionned bounds. The following proposition shows that the bound of
Theorem~\ref{refbound} is in fact tight.
\begin{proposition}
\label{tightbound}
For all $\epsilon \ge 0$, $\Delta \ge 0$, and $k>0$, there exists a $k+1$-state MDP, an initial value $v_0$ such that $\spn{v_*-v_0}=\Delta$, a sequence of noise terms $(\epsilon_j)$ with $\spn{\epsilon_j} \le \epsilon$, such that running Value Iteration during iterations with errors $(\epsilon_j)$ outputs a value function $v_{k-1}$ of which a greedy policy $\pi_k$ satisfies Equation~\eqref{bound1} with equality.
\end{proposition}
\begin{figure}
\begin{center}
\includegraphics[width=0.8\textwidth]{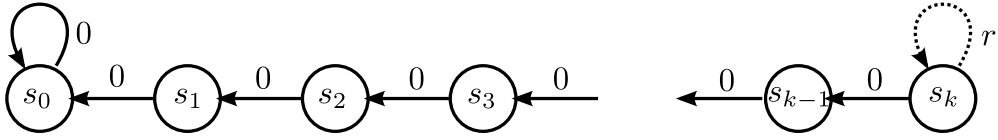}
\end{center}
\caption{\label{fig:mdp}The deterministic MDP used in the proof of Proposition~\ref{tightbound}}
\end{figure}
\begin{proof}
Consider the deterministic MDP of Figure~\ref{fig:mdp}. The only decision is in state $s_k$, where one can stay with reward $r=-\frac{\gamma-\gamma^k}{1-\gamma}\epsilon-\gamma^{k} \Delta$ or move to $s_{k-1}$ with $0$ reward. All other transitions give $0$ reward. Thus, there are only two policies, the optimal policy $\pi_*$ with value equal to $0$, and a policy $\bar \pi$ for which the value in $s_k$ is $\frac{r}{1-\gamma}$. Take
\begin{equation}
v_0(s_l)=\left\{
\begin{array}{ll}
-\Delta & \mbox{if }l=0 \\
0 & \mbox{else}
\end{array}
\right.
\mbox{~~~and for all $j<k$,~~}
\epsilon_j(s_l)=\left\{
\begin{array}{ll}
-\epsilon & \mbox{if }j=l \\
0 & \mbox{else.}
\end{array}
\right.
\end{equation}
By induction, it can be seen that for all $j \in \{1,2, \dots, k-1\}$,
\begin{align}
v_j(s_l) = \left\{
\begin{array}{ll}
-\epsilon-\gamma\epsilon-\dots-\gamma^{j-1}\epsilon-\gamma^j \Delta = -\frac{1-\gamma^{j}}{1-\gamma}\epsilon-\gamma^j \Delta & \mbox{if }j=l \\
0 & \mbox{if }j<l \le k 
\end{array} 
\right. 
\end{align}
Since $\gamma v_{k-1}(s_{k-1})=r$ and $v_{k-1}(s_k)=0$, both policies are greedy with respect to $v_{k-1}$, and the bound of Equation~\eqref{bound1} holds with equality for $\bar \pi$.
\end{proof}

Instead of running the last stationary policy $\pi_k$, one may consider running a periodic non-stationary policy, which is made of the last $m$ policies. The following theorem shows that it is indeed a good idea.
\begin{theorem} 
\label{main}
Let $\pi_{k,m}$ be the following policy
\begin{align}
\pi_{k,m} = \pi_k ~ \pi_{k-1} ~ \cdots ~ \pi_{k-m+1} ~ \pi_k ~ \pi_{k-1} ~ \cdots.
\end{align}
Then its performance loss is bounded as follows:
\begin{align}
\|v_*-v_{\pi_{k,m}}\|_\infty \le \frac{1}{1-\gamma^m} \left( \frac{\gamma-\gamma^k}{1-\gamma} \epsilon + \gamma^k \spn{v_*-v_0} \right).
\end{align}
\end{theorem}
When $m=1$, one exactly recovers the result of Theorem~\ref{refbound}. For general $m$, this new bound is a factor $\frac{1-\gamma}{1-\gamma^m}$ better than the usual bound. Taking $m=k$, that is considering all the policies generated from the very start, one obtains  the following bound:
\begin{align}
\|v_*-v_{\pi_{k,k}}\|_\infty \le \left( \frac{\gamma}{1-\gamma}-\frac{\gamma^k}{1-\gamma^k} \right) \epsilon + \frac{\gamma^k}{1-\gamma^k} \spn{v_*-v_0}.
\end{align}
that tends to $\frac{\gamma}{1-\gamma}\epsilon$ when $k$ tends to $\infty$.
In other words, we can see here that the problem of ``computing a (non stationary) approximately-optimal policy'' is not harder than that of ``computing approximately the value of some fixed policy''. Since the respective asymptotic errors are $\frac{\gamma}{1-\gamma}\epsilon$ and $\frac{1}{1-\gamma}\epsilon$, it seems even simpler~!

\begin{proof}[Proof of Theorem~\ref{main}]
The value of $\pi_{k,m}$ satisfies:
\begin{align}
\label{eqa}
v_{\pi_{k,m}} = T_{\pi_k} T_{\pi_{k-1}} \cdots T_{\pi_{k-m+1}} v_{\pi_{k,m}}.
\end{align}
By induction, it can be shown that the sequence of values generated by the algorithm satisfies:
\begin{align}
\label{eqb}
T_{\pi_k} v_{k-1} = T_{\pi_{k}} T_{\pi_{k-1}} \cdots T_{\pi_{k-m+1}} v_{k-m} + \sum_{i=1}^{m-1} \Gamma_{k,i} \epsilon_{k-i}
\end{align}
where
$$
\Gamma_{k,i}=P_{\pi_k} P_{\pi_{k-1}} \cdots P_{\pi_{k-i+1}}
$$
in which, for all $\pi$, $P_\pi$ denotes the stochastic matrix associated to policy $\pi$.
By substracting Equations~\eqref{eqb} and \eqref{eqa}, one obtains:
\begin{align}
T_{\pi_k}v_{k-1} - v_{\pi_{k,m}} &= \Gamma_{k,m} (v_{k-m}-v_{\pi_{k,m}}) + \sum_{i=1}^{m-1} \Gamma_{k,i} \epsilon_{k-i}
\end{align}
and by taking the norm
\begin{align}
\|T_{\pi_k}v_{k-1}-v_{\pi_{k,m}}\|_\infty &= \gamma^m \|v_{k-m}- v_{\pi_{k,m}} \|_\infty + \frac{\gamma-\gamma^m}{1-\gamma} \epsilon_\infty \label{eq0}
\end{align}
where $\epsilon_\infty=\max_{1 \le j \le k} \|\epsilon_j\|_\infty$.
Essentially, Equation~\eqref{eq0} shows that for sufficiently big $m$, $T_{\pi_k}v_{k-1}$ is an $\frac{\gamma}{1-\gamma}\epsilon$ approximation of the value of the non-stationary policy $\pi_{k,m}$ (whereas in general, it may be a much poorer approximation of the value of the stationary policy $\pi_k$.

By induction, it can also be proved that
\begin{align}
\|v_*-v_k\|_\infty \le \gamma^k \|v_*-v_0\|_\infty + \frac{1-\gamma^k}{1-\gamma} \epsilon_\infty. \label{eq1}
\end{align}
Using the fact that $\|T_{\pi_*}v_* - T_{\pi_k}v_{k-1} \|_\infty\le \gamma \| v_*-v_{k-1}\|_\infty $ since $\pi_*$ (resp. $\pi_k$) is greedy with respect to $v_*$ (resp. $v_{k-1}$), as well as Equations~\eqref{eq0} and~\eqref{eq1}, we can conclude by observing that
\begin{align}
\|v_* - v_{\pi_{k,m}} \|_\infty &\le \|T_{\pi_*}v_* - T_{\pi_k}v_{k-1} \|_\infty + \|T_{\pi_k}v_{k-1} - v_{\pi_{k,m}} \|_\infty \\
&  \le  \gamma \|v_*-v_{k-1}\|_\infty  + \gamma^m \|v_{k-m} - v_{\pi_{k,m}}  \|_\infty + \frac{\gamma-\gamma^m}{1-\gamma} \epsilon_\infty \\
& \le \gamma \left( \gamma^{k-1} \|v_*-v_0\|_\infty + \frac{1-\gamma^{k-1}}{1-\gamma} \epsilon_\infty\right) + \gamma^m \left( \| v_{k-m}-v_* \|_\infty  + \|v_* - v_{\pi_{k,m}} \|_\infty \right) + \frac{\gamma-\gamma^m}{1-\gamma} \epsilon_\infty \\
& \le \gamma^k \|v_*-v_0\|_\infty + \frac{\gamma-\gamma^k}{1-\gamma} \epsilon_\infty \\
& ~~~~~~~ +  \gamma^m \left( \gamma^{k-m}  \|v_*-v_0\|_\infty + \frac{1-\gamma^{k-m}}{1-\gamma} \epsilon_\infty + \|v_* - v_{\pi_{k,m}} \|_\infty  \right) + \frac{\gamma-\gamma^m}{1-\gamma} \epsilon_\infty \\
& = \gamma^m \|v_* - v_{\pi_{k,m}} \|_\infty + 2\gamma^k \|v_*-v_0\|_\infty + \frac{2(\gamma-\gamma^k)}{1-\gamma}\epsilon_\infty. 
\end{align}
Adding a constant to the value $v_j$ at any step $j$ of the algorithm does not affect the  greedy policy set $\G v_j$ and only adds a constant to the next value $v_{j+1}$. As a consequence, we can assume witout loss of generality that $\spn{v_*-v_0}=2\|v_*-v_0\|_\infty$, $\spn{\epsilon_j}=2\|\epsilon_j\|_\infty$ and the result follows.
\end{proof}

From a bibliographical point of view, the idea of using non-stationary policies to improve error bounds already appears in \cite{kakade}. However, in these works, the author considers \emph{finite-horizon} problems where the policy to be computed is naturally non-stationary. The fact that non-stationary policies (that loop over the last $m$ computed policies) can also be useful in an infinite horizon context is to our knowledge new.

\paragraph{Acknowledgements}
I thank Boris Lesner for pointing out a flaw in a previous temptative proof of Proposition~\ref{tightbound}.

\bibliographystyle{plain}
\bibliography{./biblio}

\end{document}